\documentclass[letterpaper,fleqn]{article}

\usepackage{times}
\usepackage{helvet}
\usepackage{courier}
\usepackage{placeins}
\usepackage{graphicx,picture}
\usepackage{amssymb,amsmath,amsthm}

\usepackage{algorithmic,algorithm2e}
\usepackage{mathtools}

\def\mba{\mathbf{a}}
\def\dR{\mathbb{R}}
\def\cA{\mathcal{A}}
\def\cM{\mathcal{M}}
\def\cP{\mathcal{P}}
\def\cS{\mathcal{S}}
\def\cT{\mathcal{T}}
\def\Qe{Q^{\rm exec}}

\newdimen\paravsp \paravsp=1.3ex
\def\paradot#1{\vspace{\paravsp plus 0.5\paravsp minus 0.5\paravsp}\noindent{\bf\boldmath{#1}}}

\RestyleAlgo{boxruled}

\newtheorem{definition}{Definition}
\newtheorem{proposition}{Proposition}

\def\argmax{\mathop{\arg\max}}

\def\E{{\mathbb E}} 
\def\R{{\cal R}}

\def\A{{\cal A}}

\def\S{{\cal S}}
\def\R{{\cal R}}

\def\a{{\bf a}}

\begin{document}

\title{Deep Reinforcement Learning with Attention for\\Slate Markov Decision Processes with High-Dimensional States and Actions}
\author{Peter Sunehag, Richard Evans, Gabriel Dulac-Arnold,\\ Yori Zwols, Daniel Visentin and Ben Coppin\\Google DeepMind\\London UK\\sunehag@google.com} 
\maketitle
\begin{abstract}
Many real-world problems come with action spaces represented as feature vectors. Although high-dimensional control is a largely unsolved problem, there has recently been progress for modest dimensionalities. Here we report on a successful attempt at addressing problems of dimensionality as high as $2000$, of a particular form. Motivated by important applications such as recommendation systems that do not fit the standard reinforcement learning frameworks, we introduce Slate Markov Decision Processes (slate-MDPs). A Slate-MDP is an MDP with a combinatorial action space consisting of slates (tuples) of primitive actions of which one is executed in an underlying MDP. The agent does not control the choice of this executed action and the action might not even be from the slate,  e.g., for recommendation systems for which all recommendations can be ignored. We use deep Q-learning based on feature representations of both the state and action to learn the value of whole slates. Unlike existing methods, we optimize for both the combinatorial and sequential aspects of our tasks.  The new agent's superiority over agents that either ignore the combinatorial or sequential long-term value aspect is demonstrated on a range of environments with dynamics from a real-world recommendation system. Further, we use deep deterministic policy gradients to learn a policy that for each position of the slate, guides attention towards the part of the action space in which the value is the highest and we only evaluate actions in this area. The attention is used within a sequentially greedy procedure leveraging submodularity. Finally, we show how introducing risk-seeking can dramatically improve the agents performance and ability to discover more far reaching strategies.
\end{abstract}

\section{Introduction}
Reinforcement Learning (RL) \cite {SB98} is a paradigm for learning through trial-and-error while interacting with an unknown environment. The interaction happens in cycles during which the agent chooses an action and the environment returns an observation together with a real-valued reward. The agent's goal is to maximize long-term accumulated reward. RL has had many successes, including autonomous helicopter control \cite{Ng04} and, recently, mastering a wide range of Atari games \cite{dqn15} (with a single agent) and a range of physics control tasks \cite{ddpg}. 

Although these are impressive accomplishments, the Atari games only contain $18$ actions and while the physics control tasks have continuous action spaces they are of limited dimensionality (below $10$). Our work addresses combinatorial action spaces represented by feature vectors of dimensionality up to $2000$, but with useful extra structure naturally present in the applications of interest. We consider the application of RL to problems such as  recommendation systems \cite{Park12} in which a whole slate (tuple) of actions is chosen at each time point. While these problems can be modeled with MDPs with combinatorial action spaces, the extra structure that is naturally present in the applications allows for tractable approximate value maximization of the slate.

\paradot{Slate Markov Decision Processes (Figure 1)}
We address RL problems that are such that at each time point the agent picks a fixed number of actions from a finite set $\A$. We refer to a tuple of actions as a slate. The slates are ordered in our formalization, but as a special case one can have an environment that is invariant to this order. In our environments only one action from the slate is executed. For example, in the recommendation system case, the user's choice when given the recommendations, is the execution of an action. We assume that we are given an  underlying traditional RL problem, e.g.\ a Markov Decision Process (MDP) \cite{puterman}.
\begin{figure}\label{slates}
\begin{center}
\includegraphics[width=0.5\textwidth]{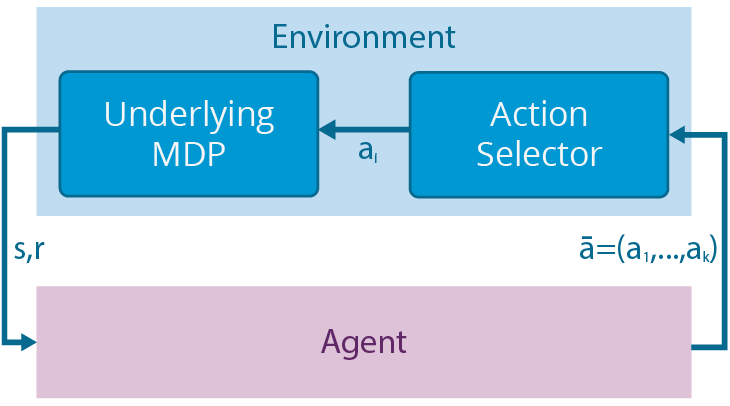}
\end{center}
 \caption{Slate-MDP Agent-Environment Framework.}
\end{figure}

 In an MDP, we observe a state $s_t$ at each time point $t$ and an action $a_t$  is executed; the next state $s_{t+1}$ and reward $r_t$ (here non-negative) received are independent of what happened before time $t$. In other words, we assume that $s_t$ summarizes all relevant information up to time $t$.

The key point of slate-MDPs is that, instead of taking one action, the agent chooses a whole slate of actions and then the environment chooses which one to execute in the underlying MDP; see Figure 1. The state information received tells the agent what action was executed, but not what would have happened for other actions. Slate-MDPs have important extra structure compared to the situations in which all the actions are executed and each full slate is its own discrete action in an enormous action space. 

 We investigate model-free agents that directly learn the slate-MDP value function for full slates. A simpler approach that is often deployed in large scale applications is to learn the values of individual actions and then combine the individually best actions. We present experiments that show  serious shortcomings of this simple approach that completely ignores the combinatorial aspect of the tasks. When extra actions are added to a slate these might interfere with the execution of the highest value action. An agent that learns a slate value function is less harmed by this and can in principle learn beneficial slate patterns such as diversity without, unlike methods such as maximum marginal relevance \cite{Carbonell98} in information retrieval, being given a mathematical definition and a constant specifying the amount of diversity to introduce.

The main drawback of full slate agents is the number of evaluations of the value function needed for producing a slate based on them. Therefore, we also investigate the option of learning a parameterized policy (a neural network) using deterministic policy gradients  \cite{silver14,ddpg,Gabe} to guide attention towards areas of the action space in which the value function is high. The neural network policy is combined with a nearest-neighbor lookup and an evaluation of the value function on this restricted set.

\paradot{Related work} As far as we are aware, \cite{Fard11} is the only work that picks multiple actions at a time for an MDP. They studied known MDPs, aiming to provide as many actions as possible while remaining nearly optimal in the worst case. Besides that they assume a known MDP, their work differs critically from our article in that they always execute an action from the slate and that they focus on the worst case choice. We work with action-execution probabilities and do not assume that any action from the slate will be executed.  Further, we work with high-dimensional feature representations and aim for a scale at which achieving guaranteed near-optimal worst case behavior is not feasible. Other work on slate actions \cite{Kale10,Yue11,Kveton14} has focused on the bandit setting in which there are no state transitions. In these articles, rewards are received and summed for each action in the slate. In our slate-MDPs, the reward is only received for the executed action and we do not know what the reward would have been for the other actions in the slate.  In the recommendation systems literature \cite{Park12}, the focus of most work is on the immediate probability of having a recommendation accepted or being relevant, and not on expected value, which is successfully optimized for here. \cite{Sha05} is an exception and optimizes for long-term value within an MDP framework, but treats individual recommendations as independent just like the agents we employ as a baseline and outperform with our full slate agents. \cite{Hado09} used continuous control methods for discrete-action problems for which the actions are embedded in a feature space and the nearest discrete action is executed. This continuous control for discrete reinforcement learning approach is in a different way (as we present at greater length for RL with large action spaces in \cite{Gabe}) utilized here as an attention mechanism.

\section{Reinforcement Learning with Slate Actions} \label{back}
As is common in reinforcement learning \cite{SB98}, we work with an agent-environment framework \cite{RN10} in which an agent chooses actions $a\in\mathcal {A}$ executed in an environment and  receives observations $ o\in\mathcal{O} $ and real-valued rewards $r\in\mathbb{R}$. A sequence of such interactions is generated cyclically at times $t=1,2,3,\hdots $. If the Markov property $\Pr(o_t,r_t\ |\ a_1,o_1,r_1, \hdots, a_{t-1},o_{t-1},r_{t-1},a_t ) = \Pr(o_t,r_t\ |\ o_{t-1},a_t)$ is satisfied, then the environment is called a Markov Decision Process (MDP) \cite{puterman} and the observations $o_t$ can be viewed as states $s_t$. 

An MDP is defined by a tuple $ \langle \mathcal{S,A,T},\R \rangle $ of a state space $\mathcal{S}$, an action space $\mathcal{A}$, a reward function $\R : \mathcal{S \times A\times S }\to \mathcal{P}(\mathbb{R})$ where $\mathcal{P(X)}$ denotes the probability distributions over the set $\mathcal{X}$, and a transition function $\mathcal{T}:\mathcal{S\times A} \rightarrow \mathcal{P(S)}$. We will use $\bar{\R}$ to denote the expected value of $\R$. A stationary policy is a function $\pi:\mathcal{S}\to\mathcal{P(A)}$ from states to probability distributions over $\mathcal{A}$.

The agent is designed to accumulate as much reward as possible, which is often addressed by aiming for maximizing expected discounted reward for a discount factor $\gamma$. In this article we primarily work with episodic environments where the discounted rewards are summed to the end of the episode. Our agents learn $Q^\pi(s,a) = \E \{\R_t|s_t=s,a_t=a,\pi\}$ for a policy $\pi$, where $\R_t = \sum_{j=0}^{t_{end}-t}{\gamma^jr_{t+j}}$ and $t_{end}$ is the end of the episode. Ideally, we want to find $Q^*(s,a) = \max_\pi Q^\pi(s,a)$ because then acting according to $\pi(s)=\argmax_a Q^*(s,a)$ is optimal. One method for achieving this is Q-learning which is based on the update $Q_{i+1}(s_t,a_t)=(1-\eta_t) Q_i(s_t,a_t)+\eta_t(r_t+\gamma \max_a Q_i(s_{t+1},a))$, where $\eta_t\in (0,1)$ is the learning rate. The update translates to a parameter update rule also for the case with a parameterized instead of a tabular value function. We employ the $\varepsilon$-greedy approach to action selection based on a value function, which means that with $1-\varepsilon$ probability we pick $\argmax_a Q_i(s,a)$ and with probability $\varepsilon$ a random action. Our study focuses on deep architectures for the value function similar to those used by \cite{dqn15,ddpg} and our approach incorporates the key techniques of target networks and experience replay employed there.

\paradot{Slate Markov Decision Processes (slate-MDPs)}
In this section we formally introduce slate-MDPs as well as some important special cases that enable more efficient inference.

\begin{definition}[slate-MDP]
Let $\cM = \langle \cS, \cA, \cT, \R \rangle$ be an MDP.
Let $\varphi\colon \cS\times \A^l \to \cA$.
Define $\cT'\colon \cS\times \A^l\times \cS \to \cP(\cS)$ and
$\R'\colon \cS\times \A^l \times \cS \to \cP(\dR)$ by
\[ \cT'(s, \mba, s') = \cT(s, \varphi(\mba), s'), \R'(s, \mba, s') = \R(s, \varphi(\mba), s'). \]
The tuple $\langle \cS, \A^l, \cT', \R' \rangle$ is
called a \emph{slate-MDP} with \emph{underlying MDP} $\cM$ and 
\emph{action-execution function} $\varphi$.
We assume that the previous executed action can be derived
from the state through a function $\psi\colon \cS\to \cA$.
\end{definition}

Note that any slate-MDP is itself an MDP with a special
structure. In particular, the probability distribution of the
next state $s'$ and the reward $r'$ conditional on the current
state $s$ and action $\mba$ can be factored as:
$$\underbrace{\Pr(s', r'| s, \mba)}_{\R'\circ \cT'} = \sum_{a\in\cA} \underbrace{\Pr(s', r' | s, a)}_{\R\circ \cT} \underbrace{\Pr(a | s, \mba)}_{\varphi}.$$
The expected reward for a slate-MDP can be computed as
$$\bar{\R}(s,\a)=\sum_{a\in\A,s'\in\S}\Pr(s'\ |\ s,a)\Pr(a\ |s,\a)\bar{\R}(s,a,s')$$ 
If we let $Q^{\rm exec}(s,a,s')=\Pr(s' |\ s,a)(\bar{\R}(s,a,s')+\gamma V^{\pi}(s'))$,
we have the following identity for the state-slate value function of the slate-MDP: 
\begin{equation}\label{eqB}
Q^{\pi}(s,\a)=\sum_{s'\in\mathcal{S},a\in\A}\Pr(a |\ s,\a)Q^{\rm exec}(s,a,s')
\end{equation}
for any slate policy $\pi:\mathcal{S}\to \mathcal{A}^l$.

We do not require that the executed action $\psi(s_{t+1})$ is an element of $\a_t$, but in the environment we create that will be the case for ``good'' slates $\a\in\mathcal{A}^l$ (Execution-Is-Best in Definition \ref{EIB}). In the recommendation system setting, $\psi(s_{t+1})\in\a_t$ means that a recommendation was selected by the user. We formally define what it means that having actions from the slate executed is the best outcome, by saying that the value-order between policies coincides with that of a modified version of the environment in which $\psi(s_{t+1})\notin \a_t$ implies that the episode ends with zero reward. We call the latter property the \emph{fatal failure} property.

\begin{definition}[Value-order, Fatal Failure, Execution-Is-Best (EIB)]\label{EIB}
Let $\mu$ and $\nu$ be two environments with the same state space $\S$ and action space $\A$.\\
{\bf Value order:} If, for any pair of policies $\pi$ and $\tilde{\pi}$,
$$V^\pi_{\mu}(s)\geq V^{\tilde{\pi}}_{\mu}(s) \iff V^\pi_{\nu}(s)\geq V^{\tilde{\pi}}_{\nu}(s)\ \forall s,$$
then we say that $\mu$ and $\nu$ have the same value-order.\\
\indent Further, suppose there is $s_{\rm end}\in\S$ such that  $\nu(s_{\rm end},r'=0 |\ s_{\rm end},\a)=1$.\\
{\bf Fatal Failure:} If $\nu(s_{\rm end},r'=0 |\ s,\a)=1$ whenever $\psi(s')\notin\a$, then we say that $\nu$ has fatal failure.\\
\indent Suppose that $\nu(s',r'\ |\ s,\a)=\mu(s',r'\ |\ s,\a)$ if $\psi(s')\in\a$, i.e., the environments coincide for executed slates.\\
{\bf EIB:} If $\nu$ has fatal failure and $\mu$ has the same value-order as $\nu$, then we say that $\mu$ has Execution-Is-Best (EIB) property.\\ 
\end{definition}

To be able to identify a value-maximizing slate in a large-scale setting, we need to avoid a combinatorial search.
The first step is to note that if an environment has the fatal failure property, then $\sum_{a\in\A}$ can be replaced by $\sum_{a\in\a}$ in \eqref{eqB}. In other words, only terms corresponding to actions in the slate are non-zero. While this condition does not hold in our environments, the EIB assumption is natural and implies that one can perform training for an environment $\nu$ modified as in Definition \ref{EIB}. Although the sum with fewer terms is easier to optimize, the problem is still combinatorial and does not scale. Therefore, monotonicity and submodularity are interesting to us since if $f\colon\cup_{j=0}^l \A^j\to\mathbb{R}$ is monotonic and submodular, we can sequentially greedily choose a slate $\a_{\rm greed}$ and $f(\a_{\rm greed})\geq (1-1/e)\max_\a f(\a)$ \cite{submod}.

\begin{definition}[Monotonic and Submodular]
We say that a function $f:\cup_{j=0}^l \mathcal{X}\to\mathbb{R}$ for $\mathcal{X}\subset \A^l$ is\\
{\bf Monotonic} if $\forall a,a_1,\hdots,a_i\in\A$ it holds that $f((a_1,\hdots,a_i,a))\geq f((a_1,\hdots,a_{i}))$ and\\ {\bf Submodular} if (diminishing returns) $$f((a_1,\hdots,a_i,a))-f((a_1,\hdots,a_{i}))\leq f((a_1,\hdots,a_{i-1},a))-f((a_1,\hdots,a_{i-1}))$$ holds for all $a,a_1,\hdots,a_i\in\A$. 
\end{definition}

To guarantee monotonicity and submodularity we introduce a further assumption that we call sequential presentation since it is satisfied if the action-selection happens sequentially in the environment, e.g, if recommendations are presented one-by-one to a user or if the users are assumed to inspect them in such a manner. Although, our environments do not have sequential presentation, the sequentially greedy procedure works well. When we evaluate the choice of a first recommendation we look at it in the presence of the other recommendations provided by a default strategy. This brings our setting closer to sequential recommendations. 

\begin{definition}[Sequential Presentation] 
We say that a slate-MDP has sequential presentation if for all states $s$ its action-execution probabilities satisfy
\begin{equation}\label{eq:seqpr}
\Pr(a | s,(a_1,\hdots,a_i,a,a_{i+1},\hdots))=\Pr(a | s,(a_1,\hdots,a_i,a))
\end{equation}
and 
$\Pr(a |\ s,(a_1,\hdots,a_i,a))\leq \Pr(a |\ s,(a_1,\hdots,a_{i-1},a)).$
\end{definition}

\begin{proposition}
If a slate-MDP has sequential presentation and satisfies the fatal failure property then its state-slate value function $Q^\pi$ is monotonic and submodular for all $\pi$.
\end{proposition}

\begin{proof}
Let $\mba_k = (a_1, \hdots, a_k)$. For any vector $\mba$ and any scalar $a$, let $\mba a$ denote the vector
constructed by concatenating $a$ to $\mba$.
Assume that $a\not\in \mba_i$. Also the rewards are nonnegative. Then, we have that $Q(s, \mba_i a) - Q(s, \mba_i) 
 =$ $$\sum_{s'\in\cS  a'\in\mba_ia} \Pr(a'|s,\mba_i a)\Qe(s, a', s') -
\sum_{s'\in\cS  a'\in\mba_i} \Pr(a'|s,\mba_i)\Qe(s, a', s')=$$ $$ 
 \sum_{s'\in\cS  a'\in\mba_i}\ 
\biggl[\Pr(a'|s,\mba_i a) - \Pr(a'|s,\mba_i) \biggr]
\Qe(s, a', s') 
+ \Pr(a|s,\mba_i a)\Qe(s, a, s') = $$ $$\Pr(a|s,\mba_i a)\Qe(s, a, s').$$
Sequential presentation immediately implies that $\Pr(a|s,\mba_i a) \leq \Pr(a|s,\mba_{i-1} a)$.
This establishes that $Q(s, \mba)$ is indeed submodular in $\mba$.  Monotonicity follows from \eqref{eq:seqpr}.
\end{proof}

The next section introduces agents based on the theory of this section. They learn the value of full slates and select a slate through a sequentially greedy procedure which under the sequential presentation assumption combined with EIB, is potentially performing slightly worse than combinatorial search. Further, motivated by EIB, training is performed on a modified environment for which fatal failure is satisfied.

\paradot{Slate agents}
\begin{algorithm}[]\label{simpleAgent}
\begin{algorithmic}[1]
\REQUIRE trainSteps, testSteps, update, $\varepsilon\geq 0$, $l\geq 1$
\STATE $t=1$, initialize $\theta$ for $Q_\theta(s,a)$
\STATE Receive initial state $s$ and take random action $a$ 
\STATE Receive reward $r$ and state $s'$. 
\REPEAT 
\STATE Pick $a'$ $\varepsilon$-greedily (slate size 1) from $Q_\theta(s',\cdot)$
\STATE Update $\theta$ using update$(s,a,r,s',a')$
\STATE  $s=s'$, $a=a'$
\STATE Perform action $a$ in environment (slate size $1$)
\STATE Receive new state $s'$ and reward $r$ 
\STATE $t=t+1$
\UNTIL $t\geq$ trainSteps

\STATE $t=1$
\STATE Receive initial state $s$ 
\REPEAT
\STATE Sort the available actions such that $Q_\theta(s,a_i)\geq Q_\theta(s,a_{i+1})\ \forall i$ 
\STATE Take slate-action $\a=(a_1,..,a_l)$ 
\STATE Receive reward $r$ and state $s'$
\STATE $t=t+1$, $s=s'$
\UNTIL $t\geq$ testSteps

\end{algorithmic}
\caption{Generic Simple (Top-$K$) Slate Agent}
\end{algorithm}
We consider model-free agents that directly learn either the value of an individual action (Algorithm \ref{simpleAgent}) or the value of a full slate (Algorithm \ref{fullAgent}). We perform the action selection for the latter in a way that only considers dependence on the actions in slots above. However, we still learn a value function which depends on a whole slate by using value function approximators that take both the features of the state and all the actions as arguments. We perform the maximization in a sequentially greedy manner and fill slots following the one being maximized with the same action that is being evaluated, while keeping previous ones fixed. Both Algorithm 1 and Algorithm 2 are stated in a generic manner while in our experiments we include useful techniques from \cite{dqn15,ddpg} that stabilize and speed up learning, namely experience replay and target networks as in Algorithm \ref{dpgalgo}. Algorithm \ref{simpleAgent} is presented in two phases; One with training using slate size $1$ and one testing with slate size $l$. In our experiments we interleave test and training phases.

\begin{algorithm}[h]\label{fullAgent}
\begin{algorithmic}[1]
\REQUIRE Steps, update, $\varepsilon\geq 0$, $l\geq 1$
\STATE $t=1$, initialize weights $\theta$ for $Q_\theta(s,\bar{a})$
\STATE Receive initial state $s$ and take random slate-action $\bar{a}$ 
\STATE Receive reward $r$ and state $s'$. 
\REPEAT 
\FOR{i=1,l}
\STATE Set $a_i=\argmax_{a\in \A(s')} Q_\theta(s', a_1,\hdots,a_{i-1},a,a,\hdots)$ 
\ENDFOR
\STATE $\a'=(a_1,\hdots,a_l)$

\STATE Update $\theta$ using update$(s,\a,r,s',\a')$
\STATE  $s=s'$, $\a=\a'$
\STATE Perform slate-action $\a$ in environment 
\STATE Receive new state $s'$ and reward $r$ 
\STATE $t=t+1$
\UNTIL $t\geq$ Steps

\end{algorithmic}
\caption{Generic Full Slate}
\end{algorithm}

\paradot{Deterministic Policy Gradient (DPG) Learning of Slate Policies to Guide Attention}
To decrease the number of evaluations of the value function when choosing a slate, we attempt to learn an attention guiding policy that is trained to produce a slate that maximizes the learned value function as seen in Algorithm \ref{dpgalgo}, which generalizes Algorithm \ref{fullAgent} in which candidate actions are used as the nearest neighbors. The policy is optimized using gradient ascent on its parameters for $Q\circ\pi$ as a function of those and the state. 

The main extra issue, besides the much higher dimensionality, compared to existing deterministic policy gradient work \cite{silver14,ddpg}, is that instead of using the continuous action produced by the neural network, we must choose from a discrete subset. We resolve this by performing a $k$-nearest neighbor lookup among the available actions and either execute the nearest or evaluate $Q$ for all the identified neighbors and pick the highest valued action. We introduce this approach in fuller detail and further developed for large action space in \cite{Gabe}. The policy is still updated in the same way since we simply want it to produce vectors with as high $Q$-values as possible. However, when we also learn $Q$ we update based on $Q(s,a)$ for the action actually taken. As in \cite{silver14,ddpg} the next action used for the TD-error $(Q(s,a)-r-\gamma Q'(s',a'))$ is the action $\pi(s')$ produced by the current target policy. To perform a nearest neighbor look-up for slates we focus on a slot at a time. We use the sequentially greedy maximization defined in Algorithm \ref{fullAgent} but for each slot the choices are further restricted to only consider the result of that look-up.

\begin{algorithm}[h]
  \caption{DPG+kNN \label{algo:ddpg}}
  \label{dpgalgo}
  \begin{algorithmic}[1]
    \STATE Randomly initialize $Q(s, a | \theta^Q)$ and policy
    $\pi(s | \theta^{\pi})$\\ with weights $\theta^{Q}$ and $\theta^{\pi}$.
    \STATE Initialize target network $Q'$ and $\pi'$ with weights $\theta^{Q'}
    \leftarrow \theta^{Q}$, $\theta^{\pi'} \leftarrow \theta^{\mu}$
    \STATE Initialize replay buffer $B$

      \STATE Receive initial observation state $s_1$
      \FOR{t = 1, T}
        \STATE With probability $1-\varepsilon$ select action $a_t$ as $\argmax_a Q(s, a | \theta^Q)$ where $a$ \\ranges across the\\ $k$ nearest candidate actions of $\pi(s_t | \theta^{\mu})$, and with probability $\varepsilon$ a random candidate action. For full\\ slate agents, $\argmax$ is replaced by sequentially greedy maximization as in Algorithm 2.
        \STATE Perform $a_t$, receive
        reward $r_t$ and new state $s_{t+1}$
        \STATE Store transition $(s_t, a_t,
                r_t, s_{t+1})$ in $B$
        \STATE Sample $(s, a,
        r, s')$ from $B$ and choose $a'$ as $a_t$\\ was chosen above but with $s'$, $Q'$ and $\pi'$.
        \STATE Set $ y = r + \gamma Q'(s',a') | \theta^{Q'}) $
            
        \STATE Update Q by gradient updates for the loss:
               $L = (y - Q(s, a | \theta^Q)^2$
        \STATE Update the policy $\pi$ using the sampled gradient:
        \begin{equation*}
            \nabla_{\theta^{\pi}} Q\circ\pi|_{s} \approx
               \nabla_{a} Q(s, a | \theta^Q)|_{\pi(s)}
               \nabla_{\theta^\pi} \pi(s | \theta^\pi)|_{s}
         \end{equation*}
        \STATE Update the target networks:
          \begin{equation*}
            \theta^{Q'} \leftarrow \tau \theta^{Q} + (1 - \tau) \theta^{Q'}
          \end{equation*}
          \begin{equation*}
            \theta^{\pi'} \leftarrow \tau \theta^{\pi} +
                (1 - \tau) \theta^{\pi'}
          \end{equation*}
        \ENDFOR
   
  \end{algorithmic}
\end{algorithm}

\section{Experimental comparison}\label{sec:exps}
We perform an experimental comparison of a range of agents described in the previous section on a test environment of a generic template. The examples used in this study, have respectively $835$, $1597$, and $13138$ states and actions represented by $100$-dimensional vectors for $835$ and $1597$ and $200$-dimensional for $13138$.

\paradot{The Test Environment:}
The test environments are such that $\A=\S=\{1,..,N\}$ and $N$ varies with the environment. An environment is defined by a transition weight matrix (from a real recommendation system) such that for each state $i\in\{1,\hdots,N\}$ and each action $j\in\{1,..,N\}$ there is a real valued non-negative weight $w_{i,j}$ indicating how common it is that $j$ follows $i$. For each $i$, only a limited number (larger than zero and at most $60$) of $w_{i,j}$ are non-zero and the magnitude of a typical non-zero weight is $0.5$. We refer to those $j$ for which $w_{i,j}>0$ as the candidate actions for state $i$. Further, there is a weight $w_{\rm fail}>0$. 

The weight matrix represents a weighted directed graph, which is extracted as a subgraph from a very large full graph of the system, by choosing a seed node and performing a breadth first traversal to a limited depth and then pruning childless nodes in an iterative manner. There is also a reward $r_j\geq 0$ for each state $j$ that is received upon transition to that state. 

When an agent in state $s$ produces a slate $\a$, each action $a\in\a$ has a probability of being executed that is proportional (not counting duplicates) to $w_a\log_2(i+1)$ (standard discount in information retrieval \cite{jarvelin02,croft10})  where $i$ is the position in the slate and the probability that no action is executed is proportional to $w_{\rm fail}$. If no action from the slate is executed, the environment transitions to a uniformly random next state and the agent receive the corresponding reward. After this transition, there is a probability (here $0.2$) of the episode ending. If $a\in\a$ was executed the environment transitions to $s'=a$, the reward is received and the episode ends with a fixed probability (here $0.1$).

\paradot{The Agents}
For all agents' $Q$-functions, we use function approximators that are feed-forward neural networks with two hidden layers, each with a $100$ units.  The policies are feed-forward neural networks with two hidden layers with $25$ hidden units each. Fewer units suffice since we only need an approximate location of high values. We use learning rate $\eta=10^{-3}$ and target network update rate $\tau=10^{-4}$. In line with theory presented in the previous section, training is performed on a modified version of the environment in which the episode ends with zero reward when an action not from the slate is executed. The update routine is a gradient step on a squared ($L_2$) loss between $Q_\theta(s_t,a_t)$ and $r_t+\gamma Q'(s_{t+1},a')$, where $Q'$ is the target $Q$-network and $a'$ the action produced by the target policy network at state $s_{t+1}$. The target network parameters $\theta'$ slowly track $\theta$ through the update $\theta'_{t+1}=(1-\tau)\theta_t'+\tau\theta$ and similarly for the target policy network. Algorithm 3 details these procedures.

\paradot{Evaluation}
We evaluate our full slate agents and simple top-$K$ agents on the three environments with slate sizes $1$, $5$ and $10$. The full slate agents are evaluated in three variations with different number of actions in the slate. The cheapest version (in number of evaluations) immediately picks the action, for each slot, whose features are nearest (in $L_2$ distance) to the vector produced by the policy. The most expensive agent considers all candidate actions and we also evaluate an agent that, for each slot, only considers the $10\%$ nearest. We ran each experiment $6$ times with different random seeds and plotted the average total reward per episode (averaged both over seeds and $1000$ episodes at evaluation with $\varepsilon=0$) in Figures 2-7 for which a moving average with window length $100$ has also been employed. The error bars show one standard deviation. Figures $2$, $3$ and $4$ compares different number of neighbors for $N=835$, $N=1597$ and $13138$. Figures $5$, $6$ and $7$ compare full and simple agents at different slate sizes for the same environments.

\paradot{Results}
We see the full slate agents performing much better overall than the simple top-$K$ agents that we employ as a baseline. The baseline is relevant since agents used in recommendation systems are often of that form (based on an unrealistic independence assumption as in \cite{Sha05}) although they typically focus on recommendations being accepted \cite{Park12}. Unlike the simple agents, full slate agents always perform well for larger slate sizes. The simple agents are unable to learn to avoid including actions with high weight but with lower value than the top pick. For slate size $1$, the simple top-$K$ agent coincides with the full slate agent which evaluates all candidate actions, hence these two agents are shown as one agent. Further, we can see that the curve for agents that only evaluate $10\%$ of the candidate actions is almost identical to the one for agents that evaluate all. The nearest neighbor agent that simply picks the nearest neighbor is slightly worse and has larger variability than the other two. However, as we demonstrate in a further experiment that also highlight the ability to learn non-myopically, the nearest neighbor agent can outperform the other agents. The variability of the nearest neighbor agent aids exploration and the attention can help when $Q$ is not estimated well everywhere. 

\paradot{Risk-Seeking}
In the case of the $N=13138$ environment in particular, it is possible to perform much better than we have already seen. In fact the performance seen in Figure 4 only reaches the performance of the optimal myopic policy. For this environment there are far better policies. We perform a simple modification to our agent (slate size $1$, all neighbors) to make it more likely to discover multi-step paths to high reward outcomes. We transform the reward that the agent is training on by replacing $r$ with $r^\alpha$, while still evaluating with the orginal reward. We see that for a wide range of exponents we eventually see far superior performance compared to $\alpha=1$. We refer to the agents with $\alpha>1$ as risk-seeking in line with prospect theory \cite{Kahneman79}.

\begin{figure}[h!]

\includegraphics[width=10cm]{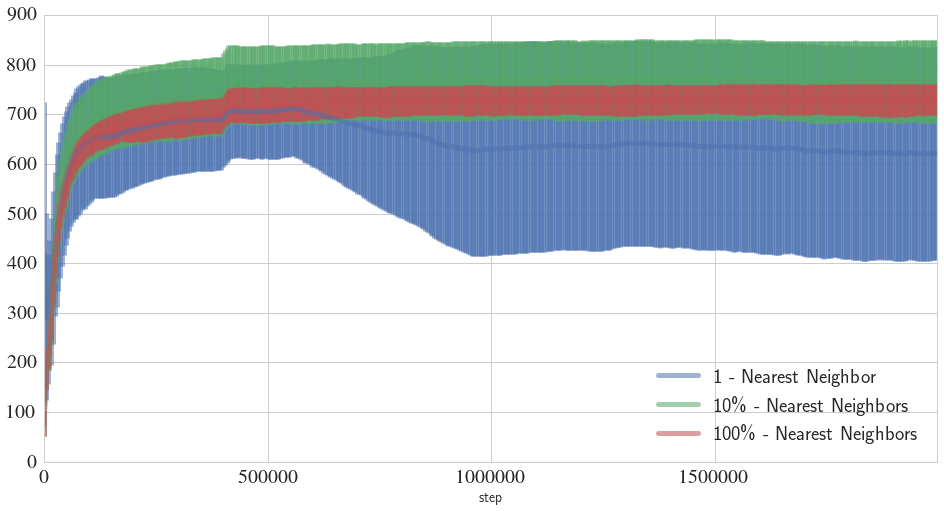}
\caption{Evaluation with different numbers of neighbors and slate size $1$ on environments with $N=835$.}
\end{figure}

\begin{figure}[h!]

\includegraphics[width=10cm]{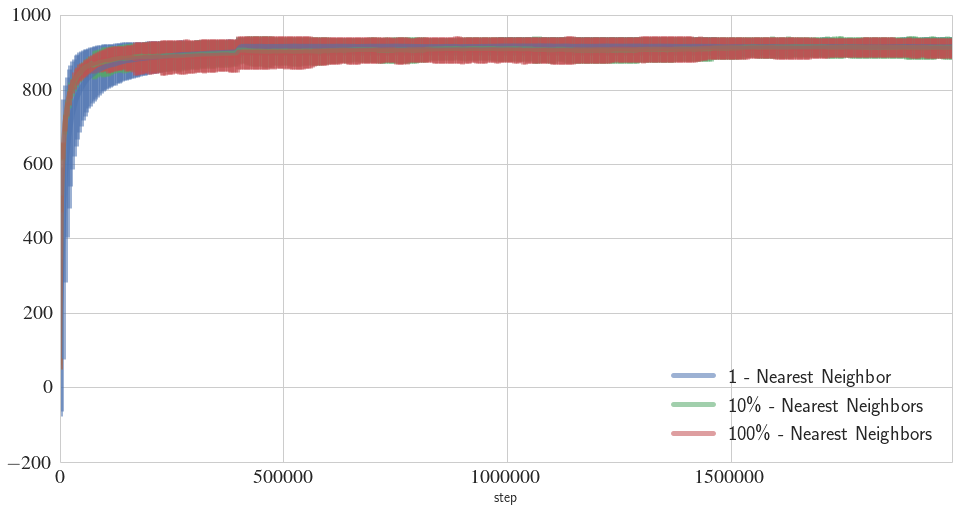}
\caption{Evaluation with different numbers of neighbors and slate size $1$ on environments with $N=1597$.}
\end{figure}

\begin{figure}[!]
\includegraphics[width=10cm]{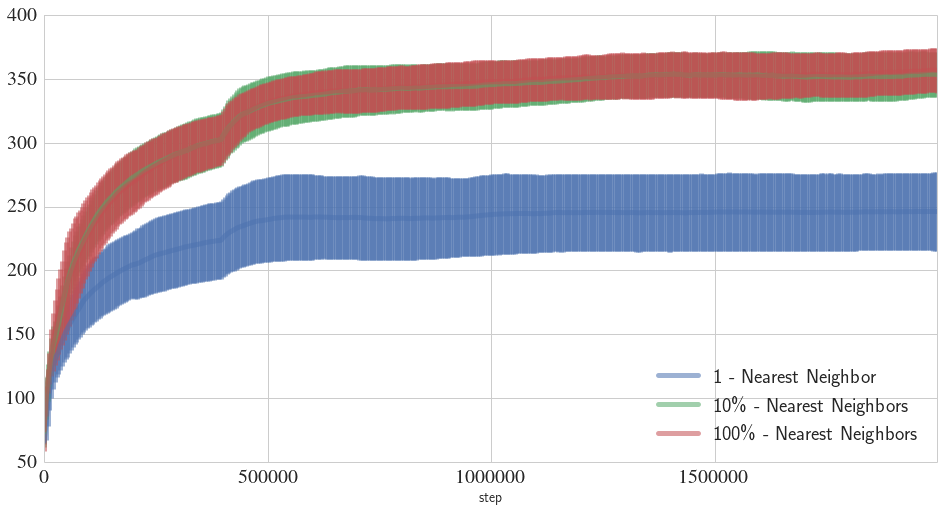}
\caption{Evaluation with different numbers of neighbors and slate size $1$ on environments with $N=13138$.}
\end{figure}

\begin{figure}[!]

\includegraphics[width=10cm]{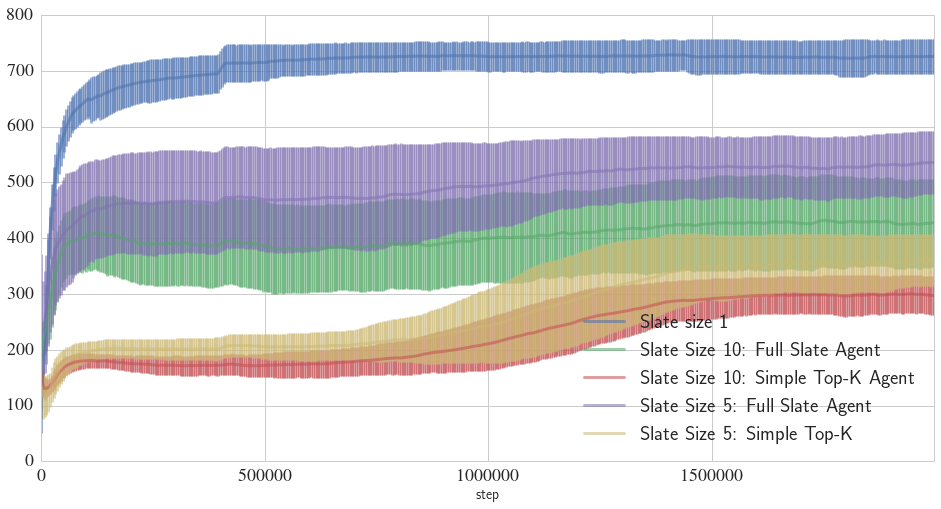}
\caption{Full and simple agents with on environments with $N=835$ and slate sizes $1$, $5$ and $10$.}
\end{figure}

\begin{figure}[!]

\includegraphics[width=10cm]{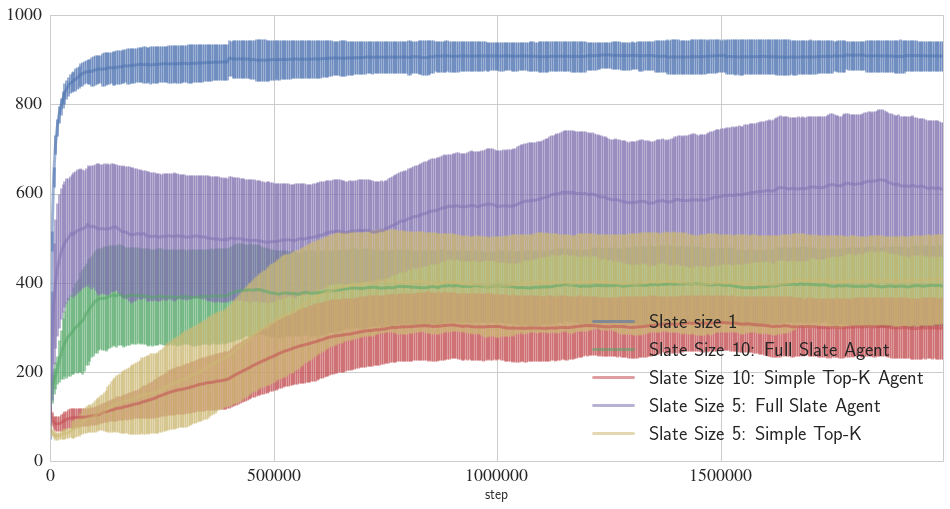}
\caption{Full and simple agents with on environments with $N=1597$ and slate sizes $1$, $5$ and $10$.}

\end{figure}

\begin{figure}[h!]
\includegraphics[width=10cm]{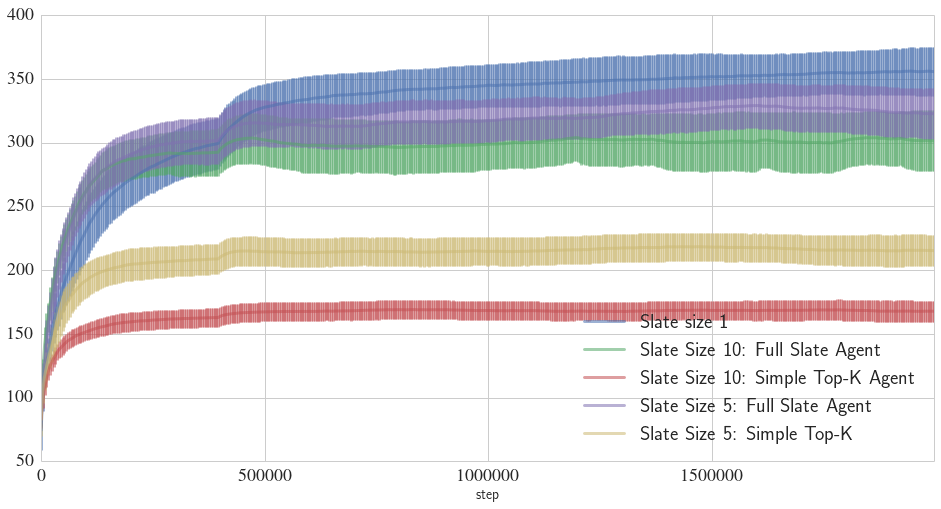}
\caption{Full and simple agents on environments with $N=13138$ and slate sizes $1$, $5$ and $10$.}
\end{figure}

\FloatBarrier

\begin{figure}[h!]
\includegraphics[width=10cm]{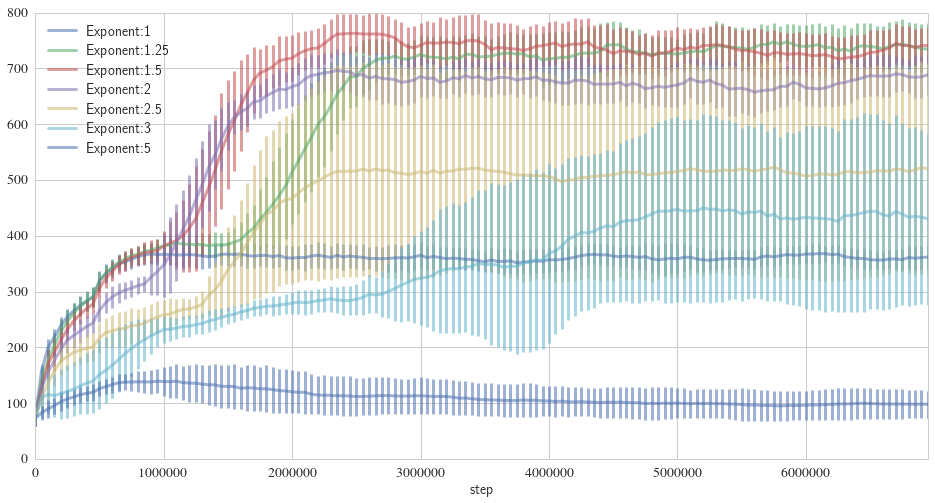}
\caption{Agent with slate size 1 evaluating all candidate actions, updating based on $r^\alpha$. Environment with $N=13138$.}
\end{figure}

\section{Conclusions}
We introduced agents that successfully address sequential decision problems with high-dimensional combinatorial slate-action spaces, found in important applications including recommendation systems. We focus on slate Markov Decision Processes introduced here, providing a formal framework for such applications. The new agents' superiority over relevant baselines was demonstrated on a range of environments derived from real world data in a live recommendation system.


\newcommand{\etalchar}[1]{$^{#1}$}

\end{document}